\documentclass{article}
\usepackage{microtype}
\usepackage{graphicx}
\usepackage{subcaption}
\usepackage{booktabs}
\usepackage{hyperref}

\usepackage[preprint]{icml2026}

\usepackage{amsmath}
\usepackage{amssymb}
\usepackage{mathtools}
\usepackage{amsthm}
\usepackage{algorithm}
\usepackage{algorithmic}
\usepackage{multirow}

\usepackage[capitalize,noabbrev]{cleveref}

\theoremstyle{plain}
\newtheorem{theorem}{Theorem}[section]
\newtheorem{proposition}[theorem]{Proposition}

\theoremstyle{definition}

\theoremstyle{remark}

\usepackage[textsize=tiny]{todonotes}

\DeclareSymbolFont{extraup}{U}{zavm}{m}{n}
\DeclareMathSymbol{\varheart}{\mathalpha}{extraup}{86}
\DeclareMathSymbol{\vardiamond}{\mathalpha}{extraup}{87}
\DeclareMathSymbol{\varclubsuit}{\mathalpha}{extraup}{88}

\icmltitlerunning{Training Data Selection with Gradient Orthogonality for Efficient Domain Adaptation}

\begin{document}

\twocolumn[
\icmltitle{Training Data Selection with Gradient Orthogonality for\\ Efficient Domain Adaptation}

    \begin{icmlauthorlist}
    Xiyang Zhang$^{{\blacklozenge\clubsuit}}$, \hspace{0.1cm}
    Yuanhe Tian$^{\varheart\clubsuit}$ \hspace{0.1cm}
    Hongzhi Wang$^{\blacklozenge}$ \hspace{0.1cm}
    Yan Song$^{{\spadesuit}}$ \\
    \vspace{0.2cm}
    $^{\blacklozenge}$Harbin Institute of Technology
    \hspace{0.1cm}
    $^{\clubsuit}$Zhongguancun Academy \\
    $^{\varheart}$Zhongguancun Institute of Artificial Intelligence \hspace{0.2cm}
    $^{\spadesuit}$University of Science and Technology of China \\
    \vspace{0.2cm}
    $^{\blacklozenge}$\texttt{s-zhangxiy24@bjzgca.edu.cn} \hspace{0.1cm}
    $^{\varheart}$\texttt{tianyuanhe@zgci.ac.cn} \\
    $^{\blacklozenge}$\texttt{wangzh@hit.edu.cn} \hspace{0.1cm}
    $^{\spadesuit}$\texttt{clksong@gmail.com} \\
    \vspace{0.4cm}
    \end{icmlauthorlist}

]

\printAffiliationsAndNotice{}  

\begin{abstract}
Fine-tuning large language models (LLMs) for specialized domains often necessitates a trade-off between acquiring domain expertise and retaining general reasoning capabilities, a phenomenon known as catastrophic forgetting. Existing remedies face a dichotomy: gradient surgery methods offer geometric safety but incur prohibitive computational costs via online projections, while efficient data selection approaches reduce overhead but remain blind to conflict-inducing gradient directions. In this paper, we propose Orthogonal Gradient Selection (OGS), a data-centric method that harmonizes domain performance, general capability retention, and training efficiency. OGS shifts the geometric insights of gradient projection from the optimizer to the data selection stage by treating data selection as a constrained decision-making process. By leveraging a lightweight Navigator model and reinforcement learning techniques, OGS dynamically identifies training samples whose gradients are orthogonal to a general-knowledge anchor. This approach ensures naturally safe updates for target models without modifying the optimizer or incurring runtime projection costs. Experiments across medical, legal, and financial domains demonstrate that OGS achieves excellent results, significantly improving domain performance and training efficiency while maintaining or even enhancing performance on general tasks such as GSM8K.
\end{abstract}

\section{Introduction}
\label{sec:intro}

The remarkable success of large language models (LLMs) has catalyzed their deployment across high-stakes vertical domains, from clinical decision support \cite{singhal2023large,liu2024bootstrapping,tian-2025-feature} to legal document analysis \cite{colombo2024saullm,katz2024gpt,colombo2024saullm} and financial advisory systems \cite{wang2023fingpt,wu2023bloomberggpt,xie2023pixiu}. However, domain-adapted LLMs face a fundamental paradox: the very process of specialization tends to erode the general-purpose capabilities that made these models valuable in the first place. Practitioners face what amounts to a zero-sum game between specialists and generalists, which means optimizing for medical question answering degrades mathematical reasoning, while preserving arithmetic skills constrains domain learning \cite{luo2025empirical}. This phenomenon, known as catastrophic forgetting \cite{mccloskey1989catastrophic}, has emerged as a central bottleneck in the efficient and practical deployment of domain-adapted LLMs.

The geometric roots of this problem lie in gradient interference \cite{yu2020gradient}. When a model updates its parameters to minimize domain-specific loss, the resulting gradient may point in a direction that increases the loss on general tasks. Formally, forgetting occurs when the inner product of the domain task's gradient and the general task's gradient is less than zero \cite{lopez2017gradient}, causing the model to drift away from regions of the parameter space that support general reasoning. This manifold drift accumulates over thousands of training steps, ultimately severing the model's connection to its pre-trained knowledge \cite{kumar2022fine}, shown as catastrophic forgetting.

Prior work has attempted to resolve this conflict from two distinct angles, yet both suffer from limitations that force a compromise between effectiveness and efficiency. The first family, optimizer-level gradient surgery such as GEM \cite{lopez2017gradient} and SafeGrad \cite{yi2025gradient}, actively projects task gradients onto a safe subspace orthogonal to protected knowledge \cite{saha2021gradient}. While theoretically principled, their computational demands scale poorly: each optimization step requires computing reference gradients and executing high-dimensional projections \cite{chaudhry2018efficient}. For LLMs exceeding dozens of billion parameters, this online overhead becomes prohibitive, rendering such methods impractical for scale-efficient training \cite{shi2025continual}. We need a faster way.

The second route embraces data selection as a scalable alternative \cite{song-etal-2012-entropy,coleman2019selection,liu2019reinforced,killamsetty2021glister}. Methods such as LESS \cite{xia2024less} and GrADS \cite{liu2025learn} improve efficiency by filtering data offline based on gradient similarity or magnitude. However, these approaches often remain blind to the geometric nature of forgetting \cite{lopez2017gradient}. LESS focuses on target relevance without assessing conflicts with general knowledge, while GrADS relies on gradient norms which cannot distinguish between orthogonal (safe) and opposing (harmful) updates. Consequently, these methods may inadvertently select data that maximizes domain gain at the cost of severe forgetting, failing to solve the multi-objective challenge.

This analysis reveals a critical methodological gap: we need a solution that possesses the geometric precision of gradient surgery but maintains the computational efficiency of data selection. Considering that the small proxy models can guide the LLMs \cite{burns2023weak, xie2023data}, the natural question emerges: \emph{Can we employ proxy model and intelligent decision-making to transfer the geometric insights of gradient projection to the data selection phase?}

To this end, we propose \textbf{Orthogonal Gradient Selection (OGS)}, a novel data-centric method designed to simultaneously boost domain performance, prevent catastrophic forgetting, and maximize training efficiency. OGS fundamentally rethinks data selection not as a static filtering task, but as a constrained optimization problem guided by gradient geometry. Our approach relies on three key innovations:

\textbf{1. Offline Geometric Awareness:} We construct a \emph{general-knowledge anchor} whose gradient defines the direction to protect. Instead of projecting gradients during training, we select data samples that naturally produce gradients orthogonal to this anchor, effectively performing ``data surgery'' to ensure safe updates in a efficient way.

\textbf{2. Navigator-Target Architecture:} To bypass the cost of computing gradients on LLMs, we introduce a lightweight ``Navigator'' proxy model. We demonstrate that gradient geometry features (orthogonality and conflict) computed on the Navigator transfer reliably to much larger Target models, allowing us to perform expensive geometric analysis.

\textbf{3. RL-Driven Selection Policy:} We formulate the selection process as a decision-making task solvable via reinforcement learning (RL) mechanisms. By optimizing a reward function that balances domain learning speed with orthogonality constraints, OGS dynamically curates a curriculum that navigates the trade-off between plasticity (learning new tasks) and stability (remembering old ones).

The paradigm shift to OGS carries profound practical implications. By decoupling geometric analysis from the training loop, OGS remains compatible with standard optimizers and efficient training pipelines such as LoRA \cite{hu2022lora}, incurring negligible runtime overhead. Theoretical analysis confirms that our selection criterion acts as a first-order approximation to the optimal solution of a bilevel optimization problem minimizing post-update general loss.

Empirically, we evaluate OGS across three vertical domains (medical, law, and finance) using target models ranging from 1.7B to 14B parameters. The results demonstrate that OGS yields Pareto-optimal improvements: it exceeds the domain accuracy of baselines while preserving pre-training performance on general benchmarks. Crucially, OGS achieves this with higher training throughput than gradient surgery methods, validating its potential as a general-purpose method for efficient, safe, and effective domain adaptation.

In summary, we make the following contributions:
\begin{itemize}
    \item We propose OGS, a data-centric method that leverages gradient geometry to simultaneously enhance domain adaptation and prevent catastrophic forgetting, offering a scalable alternative to online gradient surgery.
    \item We introduce a Navigator-Target architecture coupled with a reinforcement learning-based selection policy, enabling the efficient transfer of geometric insights from small proxy models to large-scale target models.
    \item We provide theoretical grounding for OGS via an interference decomposition theorem and establish its selection rule as a solution to a constrained bilevel optimization problem with detailed proof.
    \item We demonstrate that OGS achieves superior efficiency across multiple domains, improving domain performance and training speed while effectively neutralizing the risk of catastrophic forgetting.
\end{itemize}

\section{Related Work}
\label{sec:related}

\subsection{Continual Learning and Catastrophic Forgetting in LLMs}
Adapting LLMs to vertical domains often necessitates a trade-off between acquiring domain-specific knowledge and retaining general capabilities \cite{wu2023bloomberggpt,shu2024lawllm,tian-etal-2024-chimed,colombo2024saullm,su2025fusing,liu2025balanced}, a phenomenon known as the stability-plasticity dilemma \cite{mermillod2013stability}. While conventional parameter-efficient fine-tuning (PEFT) methods mitigate forgetting by freezing the backbone \cite{houlsby2019parameter, hu2022lora}, they do not fundamentally resolve gradient conflicts between new and old tasks. Replay-based methods \cite{rebuffi2017icarl} remain the gold standard but introduce the challenge of selecting a representative replay buffer. Recent works have explored prompt-based approaches \cite{wang2022learning} or model merging \cite{ilharco2022editing, yadav2023ties} to alleviate forgetting. However, these methods often operate at the model level rather than the data level, potentially missing the root cause of interference: the training data itself \cite{zhou2023lima}.

\subsection{Gradient Surgery and Conflicting Objectives}
Gradient surgery methods aim to manipulate gradient updates during optimization to mitigate interference between tasks. Seminal works like GEM \cite{lopez2017gradient} and A-GEM \cite{chaudhry2018efficient} project gradients of new tasks onto the feasible region defined by previous tasks. PCGrad \cite{yu2020gradient} projects conflicting gradients onto the normal plane of each other to resolve interference in multi-task learning. More recently, \textbf{SafeGrad} \cite{yi2025gradient} applied this concept to safety alignment, projecting user-task gradients onto the subspace orthogonal to safety alignment gradients to prevent ``safety unlearning." While theoretically elegant, these optimizer-level solutions suffer from a fatal bottleneck: they require computing and projecting gradients for reference tasks at every optimization step. This introduces prohibitive computational and memory overhead (often $2\times$ to $3\times$ cost), making them impractical for LLMs exceeding high-level parameters \cite{abbes2025revisiting}. In contrast, OGS transfers the geometric insights of gradient surgery to the data selection stage. By identifying naturally orthogonal data offline, OGS achieves the stability benefits of gradient surgery with the efficiency of standard SFT, requiring no runtime overhead.

\subsection{Data Selection for Efficient Fine-tuning}
Data selection has emerged as a crucial paradigm for efficient LLM training. Early methods focused on quality filtering via heuristics or perplexity \cite{zhang-etal-2019-multiplex,wenzek2020ccnet, penedo2023refinedweb}. Recent gradient-based approaches seek to select high-value samples mathematically. \textbf{LESS} \cite{xia2024less} selects data that minimizes loss on a target validation set by maximizing gradient similarity. While effective for targeted induction, LESS is blind to ``safety": maximizing similarity to a target domain inadvertently selects samples that maximally conflict with general capabilities (e.g., selecting medical misconceptions that contradict general reasoning). \textbf{GrADS} \cite{liu2025learn} proposes a gradient-aware selection strategy based on gradient magnitude and uncertainty. However, gradient magnitude can be misleading \cite{wu2025imbalanced}, i.e., a sample with moderate gradient magnitude can still be highly destructive if its direction is diametrically opposed to general knowledge (i.e., negative cosine similarity).

\section{Methodology}
\label{sec:method}

We present Orthogonal Gradient Selection (OGS) in detail. We begin by formalizing the constrained optimization problem that captures the stability-plasticity trade-off, then introduce our Navigator-Target architecture that enables efficient gradient geometry computation at scale. The core of our method lies in two geometric metrics—orthogonality and conflict—that guide data selection toward the safe subspace. We conclude with the selection strategies and the constrained policy optimization method. Figure \ref{fig:method} shows the entire intuitive flow of the OGS method.

\begin{figure*}[ht]
    \centering
    \includegraphics[width=\textwidth, trim=0 50 0 0]{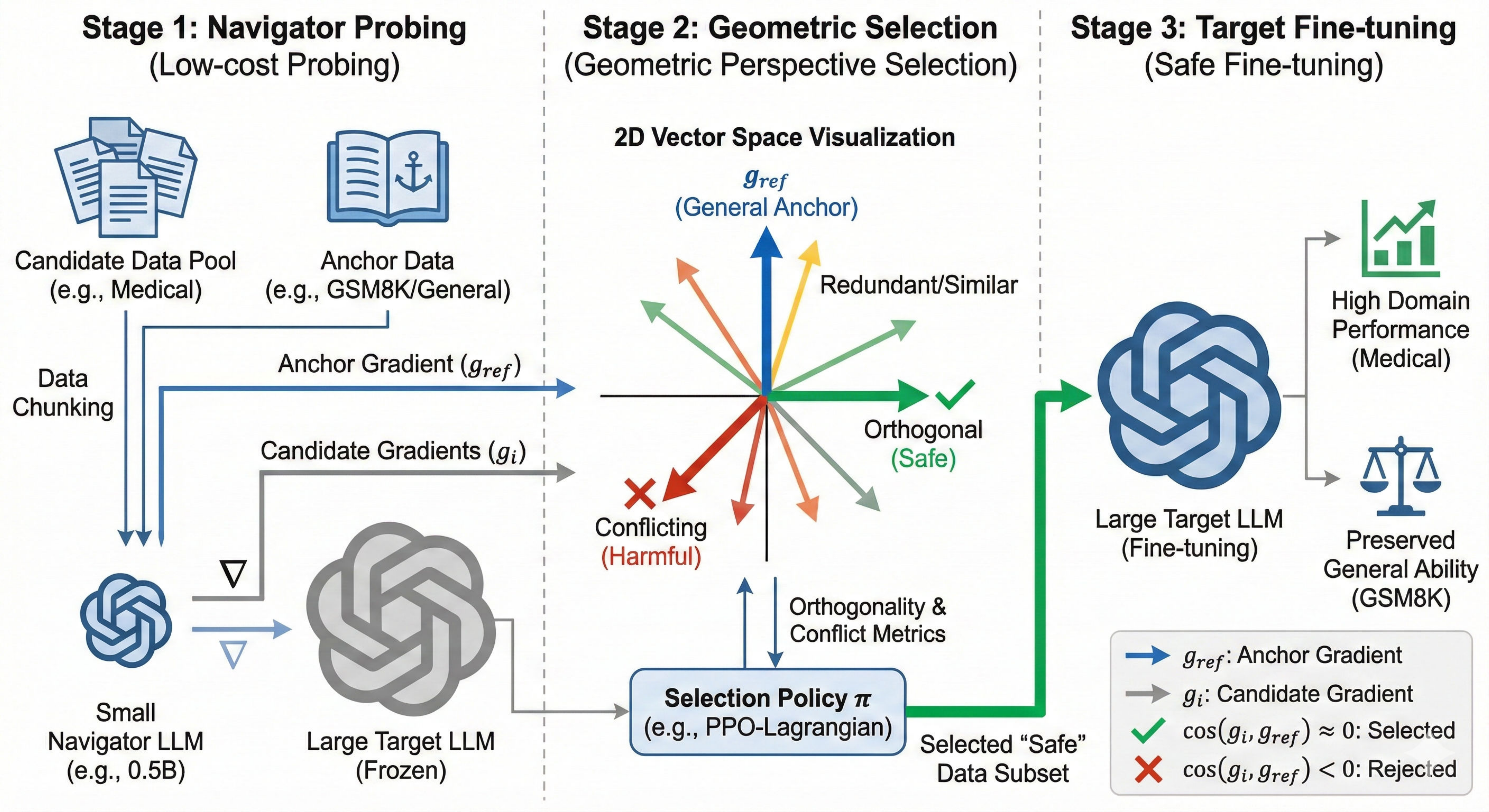}
    \caption{The entire intuitive flow of the OGS method, including Navigator Probing, Geometric Selection, Target Fine-tuning three stages, aiming to choose the best data subset.}
    \label{fig:method}
    \vspace{-4mm}
\end{figure*}

\subsection{Problem Formulation}
\label{sec:problem}

Consider a pre-trained language model with parameters $\theta$, a domain-specific training pool $\mathcal{D}_{\text{domain}}$ (e.g., medical question-answering pairs), and a general-knowledge validation set $\mathcal{V}_{\text{gen}}$ representing capabilities we wish to preserve (e.g., mathematical reasoning, world knowledge). Let $\mathcal{R}_{\text{domain}}(\theta)$ and $\mathcal{R}_{\text{gen}}(\theta)$ denote performance metrics on domain and general tasks respectively. The goal of continual domain adaptation is to maximize domain performance while maintaining general capabilities above a threshold.

We formalize this as a constrained optimization problem over data selection policies. Let $\pi$ denote a policy that, at each training step $t$, selects a batch of samples from the available data pools. The training trajectory $\tau = \{(s_t, a_t, s_{t+1})\}_{t=0}^{T-1}$ induced by $\pi$ produces final parameters $\theta_T$. Our objective is:

\begin{equation}
\label{eq:constrained_opt}
\max_{\pi} \; \mathbb{E}_{\tau \sim \pi}\left[\mathcal{R}_{\text{domain}}(\theta_T)\right] \quad \end{equation}

\begin{equation}
\text{s.t.} \quad \mathbb{E}_{\tau \sim \pi}\left[\mathcal{R}_{\text{gen}}(\theta_T)\right] \geq \mathcal{R}_{\text{gen}}(\theta_0) + \delta
\end{equation}

where $\delta \geq 0$ specifies the minimum acceptable improvement (or preservation when $\delta = 0$) of general capabilities. This formulation explicitly encodes the stability-plasticity trade-off: the objective drives domain learning while the constraint enforces capability preservation.

\subsection{Navigator-Target Architecture}
\label{sec:navigator}

A naive approach to gradient-geometry-aware selection would compute, for each candidate sample, its gradient on the target model and measure geometric relationships with anchor gradients. For a huge size parameter target model, this is computationally prohibitive, as each gradient computation requires a full forward-backward pass, consuming hundreds of gigabytes of memory.

A key insight is that gradient geometric features exhibit remarkable transferability across model scales. While the absolute gradient magnitudes differ substantially between a 0.5B Navigator and a 70B Target, the relative geometric relationships, such as cosine similarities, orthogonality patterns, and conflict structures, remain highly correlated. This phenomenon arises because models are from the same architectural family, trained on similar corpora, and develop aligned representation spaces where semantically similar samples induce geometrically similar gradient directions.

We exploit this transferability through a two-phase method. In Phase 1 (Strategy Learning), a lightweight Navigator model $\mathcal{M}_n$ with parameters $\theta_n$ (e.g., Qwen3-0.6B) serves as the computational workhorse. We compute gradient geometry features on $\mathcal{M}_n$, train a selection policy via reinforcement learning, and validate geometric predictions against ground-truth forgetting measurements. In Phase 2 (Strategy Application), the learned policy guides data selection for the Target model $\mathcal{M}_t$ (e.g., Qwen3-14B).

This architecture converts the computational burden from multiplicative (per training step on Target) to additive (one-time preprocessing on Navigator). The Navigator processes the entire candidate pool once, after which Target training proceeds at full speed without geometric overhead.

\subsection{Gradient Geometry Metrics}
\label{sec:geometry}

The foundation of OGS is a set of metrics that quantify how a candidate sample's gradient relates to protected general knowledge. We first construct an anchor that represents the gradient direction of general capabilities, then define orthogonality and conflict metrics relative to this anchor.

\paragraph{Anchor Gradient Construction.}
We curate a compact anchor dataset $\mathcal{D}_{\text{anchor}}$ comprising 300-500 exemplars from capabilities we wish to protect: mathematical reasoning problems from GSM8K, factual questions from MMLU, and instruction-following examples from Alpaca. The anchor gradient is the average gradient over this set:

\begin{equation}
\label{eq:anchor}
\mathbf{g}_{\text{ref}} = \frac{1}{|\mathcal{D}_{\text{anchor}}|} \sum_{x \in \mathcal{D}_{\text{anchor}}} \nabla_\theta \mathcal{L}(x; \theta)
\end{equation}

This vector defines the ``north star'' direction in parameter space, which means updating aligned with $\mathbf{g}_{\text{ref}}$ improves general capabilities, while updates opposing it induce forgetting. The anchor gradient is computed once and cached for the entire selection process.

For enhanced sensitivity, we optionally employ \emph{active anchor selection}, preferentially including general samples that exhibit maximum conflict with domain data:

\begin{equation}
\mathcal{D}_{\text{anchor}}^{\text{active}} = \arg\min_{D \subset \mathcal{D}_{\text{gen}}, |D|=k} \cos\left(\bar{\mathbf{g}}_D, \bar{\mathbf{g}}_{\text{domain}}\right)
\end{equation}

where $\bar{\mathbf{g}}_D$ denotes the mean gradient over subset $D$. This selects anchors representing knowledge most vulnerable to domain-induced forgetting.

\paragraph{Orthogonality Score.}
For a candidate sample $x_i$ with gradient $\mathbf{g}_i = \nabla_\theta \mathcal{L}(x_i; \theta)$, we define orthogonality as:

\begin{equation}
\label{eq:orthogonality}
\text{Orth}(x_i) = 1 - \left|\cos(\mathbf{g}_i, \mathbf{g}_{\text{ref}})\right| = 1 - \frac{|\mathbf{g}_i^\top \mathbf{g}_{\text{ref}}|}{\|\mathbf{g}_i\| \cdot \|\mathbf{g}_{\text{ref}}\|}
\end{equation}

Orthogonality ranges from 0 to 1, with higher values indicating that the sample's gradient lies closer to the hyperplane perpendicular to $\mathbf{g}_{\text{ref}}$. A sample with $\text{Orth}(x_i) \approx 1$ induces updates in the \emph{safe subspace}, which means these directions neither help nor harm general capabilities. Such samples are able to be trained freely without risking forgetting.

\paragraph{Conflict Score.}
While orthogonality measures proximity to the safe subspace, conflict directly quantifies interference:

\begin{equation}
\label{eq:conflict}
\text{Conf}(x_i) = -\cos(\mathbf{g}_i, \mathbf{g}_{\text{ref}}) = -\frac{\mathbf{g}_i^\top \mathbf{g}_{\text{ref}}}{\|\mathbf{g}_i\| \cdot \|\mathbf{g}_{\text{ref}}\|}
\end{equation}

Conflict ranges from $-1$ to $+1$. Positive conflict ($\text{Conf} > 0$) indicates that the sample's gradient opposes the anchor direction, meaning that training on this sample would increase general-task loss, inducing forgetting. Negative conflict ($\text{Conf} < 0$) indicates synergy: the sample simultaneously benefits both domain and general capabilities. Zero conflict corresponds to perfect orthogonality.

These metrics decompose gradient relationships into actionable signals. It illustrates geometric intuition: the anchor gradient $\mathbf{g}_{\text{ref}}$ defines a protected direction, and candidate gradients are classified by their angle relative to direction.

\subsection{Selection Strategies}
\label{sec:strategies}

Armed with geometric metrics, we define two complementary selection strategies that together achieve robust capability preservation. These two strategies can also be mixed.

\paragraph{Orthogonal Protection.}
The primary strategy selects samples residing within the safe subspace:

\begin{equation}
\mathcal{D}_{\text{selected}}^{\text{orth}} = \{x_i \in \mathcal{D}_{\text{domain}} : \text{Orth}(x_i) \geq \tau_{\text{orth}}\}
\end{equation}

where $\tau_{\text{orth}} \in [0, 1]$ is an orthogonality threshold. These samples enable unimpeded domain learning: their gradients update parameters in directions that leave general capabilities essentially unchanged. This strategy maximizes learning efficiency when sufficient orthogonal samples exist.

\paragraph{Conflict-Aware Replay.}
When domain training inevitably includes some conflicting samples, we counteract their harmful effects through targeted replay. We identify general-knowledge samples currently under attack:

\begin{equation}
\mathcal{D}_{\text{replay}} = \{x_j \in \mathcal{D}_{\text{gen}} : \text{Conf}(x_j) > \tau_{\text{conf}}\}
\end{equation}

These are samples whose gradients most strongly oppose recent domain updates—precisely the knowledge being forgotten. By interleaving replay of these samples, we inject corrective gradients that neutralize interference.

\paragraph{Hybrid Dynamic Strategy.}
We employ a hybrid approach that adapts the mixture ratio based on training dynamics:

\begin{equation}
\mathcal{B}_t = \alpha(t) \cdot \mathcal{B}_t^{\text{orth}} + (1 - \alpha(t)) \cdot \mathcal{B}_t^{\text{replay}}
\end{equation}

where $\alpha(t)$ evolves during training. Early phases emphasize conflict-aware replay ($\alpha$ small) to establish stability, while later phases shift toward orthogonal domain samples ($\alpha$ large) to accelerate specialization. The policy network learns to modulate $\alpha(t)$ based on observed performance dynamics to keep the balance.

\subsection{Constrained Policy Optimization}
\label{sec:optimization}

We cast data selection as a Constrained Markov Decision Process (CMDP) and solve it via PPO-Lagrangian optimization. The state $s_t$ encodes current model performance, gradient geometry statistics across data clusters, and training progress. The action $a_t$ selects a data cluster from which to sample the next batch. The reward $r_t$ reflects domain performance improvement, while the cost $c_t$ measures general capability degradation.

The constrained objective from Equation \ref{eq:constrained_opt} is converted to an unconstrained Lagrangian:

\begin{equation}
\label{eq:lagrangian}
\mathcal{L}(\pi, \lambda) = \mathbb{E}_\pi\left[\sum_t \gamma^t r_t\right] - \lambda \cdot \left(\mathbb{E}_\pi\left[\sum_t \gamma^t c_t\right] - \epsilon\right)
\end{equation}

where $\lambda \geq 0$ is the Lagrange multiplier and $\epsilon$ is the constraint budget. We alternate between policy updates (maximizing $\mathcal{L}$ with respect to $\pi$ via PPO) and dual updates (adjusting $\lambda$ based on constraint satisfaction):

\begin{equation}
\lambda_{k+1} = \max\left(0, \lambda_k + \eta_\lambda \cdot (J_C(\pi_k) - \epsilon)\right)
\end{equation}

When the policy violates constraints (excessive forgetting), $\lambda$ increases, penalizing unsafe selections more heavily. When constraints are comfortably satisfied, $\lambda$ decreases, allowing more aggressive domain learning. This adaptive mechanism automatically navigates the stability-plasticity trade-off without manual tuning.

Algorithm \ref{alg:ogs} summarizes the complete OGS pipeline. Phase 1 trains the selection policy on the Navigator, while Phase 2 applies it to the Target with standard fine-tuning. 

\begin{algorithm}[t]
\caption{Orthogonal Gradient Selection (OGS)}
\label{alg:ogs}
\begin{algorithmic}[1]
\STATE \textbf{Input:} Domain pool $\mathcal{D}_{\text{domain}}$, general pool $\mathcal{D}_{\text{gen}}$, Navigator $\mathcal{M}_n$, Target $\mathcal{M}_t$
\STATE \textbf{Output:} Fine-tuned Target parameters $\theta_t^*$

\STATE \textit{// Phase 0: Preprocessing}
\STATE Construct anchor set $\mathcal{D}_{\text{anchor}}$ from $\mathcal{D}_{\text{gen}}$
\STATE Compute anchor gradient $\mathbf{g}_{\text{ref}}$ via Eq. \ref{eq:anchor}
\STATE Cluster $\mathcal{D}_{\text{domain}}$ and $\mathcal{D}_{\text{gen}}$ into $K$ clusters

\STATE \textit{// Phase 1: Policy Learning on Navigator}
\FOR{episode $e = 1$ to $E$}
    \STATE Reset Navigator: $\theta_n \leftarrow \theta_n^{(0)}$
    \FOR{step $t = 0$ to $T-1$}
        \STATE Compute cluster geometries: $\{\text{Orth}_k, \text{Conf}_k\}_{k=1}^K$
        \STATE Construct state $s_t$ from geometries and performance
        \STATE Sample action $a_t \sim \pi_\psi(\cdot | s_t)$
        \STATE Sample batch $\mathcal{B}_t$ from cluster $a_t$
        \STATE Update Navigator: $\theta_n \leftarrow \theta_n - \eta \nabla \mathcal{L}(\mathcal{B}_t)$
        \STATE Compute reward $r_t$ and cost $c_t$
    \ENDFOR
    \STATE Update policy $\pi_\psi$ via PPO-Lagrangian
\ENDFOR

\STATE \textit{// Phase 2: Policy Application on Target}
\FOR{step $t = 0$ to $T_{\text{target}}-1$}
    \STATE Compute state $s_t$ for Target
    \STATE Select cluster $a_t = \arg\max_a \pi_\psi(a | s_t)$
    \STATE Sample batch $\mathcal{B}_t$ from cluster $a_t$
    \STATE Update Target via standard SFT: $\theta_t \leftarrow \theta_t - \eta \nabla \mathcal{L}(\mathcal{B}_t)$
\ENDFOR
\STATE \textbf{return} $\theta_t^* = \theta_t$
\end{algorithmic}
\end{algorithm}

\section{Theoretical Analysis}
\label{sec:theory}

We provide a theoretical foundation for OGS, demonstrating that our geometric constraints are necessary to minimize catastrophic forgetting and that our selection strategy is a first-order optimal solution to the constrained continual learning problem. We provide efficiency analysis as well. All detailed proofs are deferred to Appendix \ref{app:theory}.

\subsection{Geometric Interference and Safety Subspace}

Catastrophic forgetting in fine-tuning arises from the interference between domain-specific updates and general knowledge representations. Let $\theta \in \mathbb{R}^d$ be the model parameters, and $\mathcal{L}_{med}, \mathcal{L}_{gen}$ be the domain and general loss functions, respectively.

\begin{theorem}[Gradient Interference]
\label{thm:interference}
The degradation of general capabilities after a domain update step $\Delta \theta = -\eta \nabla \mathcal{L}_{med}(\theta)$ is governed, at the first order, by the inner product of gradients:
\begin{equation}
    \Delta \mathcal{L}_{gen} = -\eta \langle \nabla \mathcal{L}_{gen}(\theta), \nabla \mathcal{L}_{med}(\theta) \rangle + \mathcal{O}(\eta^2).
\end{equation}
\end{theorem}

Theorem \ref{thm:interference} dictates to prevent forgetting ($\Delta \mathcal{L}_{gen} > 0$), the domain gradient must be non-conflicting with the general gradient. Accordingly, we define the \textbf{Safety Subspace} as $\mathcal{S}_{\perp} = \{ \mathbf{v} \mid \langle \mathbf{v}, \mathbb{E}[\nabla \mathcal{L}_{gen}] \rangle \ge 0 \}$. OGS explicitly selects data gradients residing in $\mathcal{S}_{\perp}$, ensuring safety \textit{a priori}.

\subsection{Optimality and Efficiency}

We formulate valid data selection as a bilevel optimization problem: minimizing general loss degradation (outer loop) subject to selecting $k$ samples for domain maximization (inner loop).

\begin{theorem}[First-Order Optimality]
\label{thm:bilevel}
The optimal selection strategy $\mathbf{w}^*$ that solves the bilevel objective under a first-order approximation is equivalent to maximizing the alignment $\sum w_i \langle \nabla \ell(x_i), \nabla \mathcal{L}_{gen} \rangle$.
\end{theorem}

Theorem \ref{thm:bilevel} proves that our ``Conflict" metric (negative cosine similarity) is mathematically aligned with the optimal policy. Furthermore, we address the feasibility of our dual-model architecture.

\begin{proposition}[Asymptotic Efficiency]
\label{prop:complexity}
Let $\rho = M_n / M_t \ll 1$ be the parameter ratio between Navigator and Target models. The computational overhead of OGS scales with $\mathcal{O}(\rho)$, rendering it negligible compared to the $\mathcal{O}(1)$ overhead of online gradient surgery methods.
\end{proposition}

This establishes OGS as a Pareto-optimal solution, combining the theoretical rigor of gradient projection with the efficiency of standard fine-tuning.

\section{Experiments}
\label{sec:experiments}

To rigorously evaluate the efficacy of our method, we conducted extensive experiments across healthcare, law, and finance three distinct vertical domains on five LLMs. Our empirical analysis investigates following dimensions: the capacity of OGS to achieve domain adaptation performance comparable to full-data training while utilizing significantly fewer samples, and the efficacy of the geometric safety constraint in mitigating catastrophic forgetting, particularly within fragile reasoning tasks.

\begin{table*}[t]
\centering
\caption{\textbf{Main Results on 8B Scale Models (10\% Data Selection).} Comparison of domain adaptation (MedQA, LegalBench, FinQA) and general capability retention (GSM8K, MMLU, ARC-C) using Llama-3.1-8B and Qwen3-8B. The ``Avg." columns represent the macro-average across the three respective tasks. OGS achieves the best balance, significantly outperforming baselines on the critical GSM8K retention metric while matching or exceeding domain performance. \textbf{Bold} indicates the best performance among selection methods.}
\label{tab:main_8b_results}
\resizebox{\textwidth}{!}{%
\begin{tabular}{llcccccccc}
\toprule
\multicolumn{1}{c}{\multirow{2}{*}{\textbf{Model}}} & \multicolumn{1}{c}{\multirow{2}{*}{\textbf{Method}}} & \multicolumn{4}{c}{\textbf{Domain Performance}} & \multicolumn{4}{c}{\textbf{General Capability}} \\
\cmidrule(lr){3-6} \cmidrule(lr){7-10}
 & & \textsc{MedQA} & \textsc{Legal} & \textsc{FinQA} & \textbf{Avg.} & \textsc{GSM8K} & \textsc{MMLU} & \textsc{ARC-C} & \textbf{Avg.} \\
\midrule
\multirow{7}{*}{\textbf{Llama-3.1-8B}} 
 & Full Data (100\%) & \textbf{65.2} & \textbf{58.4} & \textbf{61.3} & \textbf{61.6} & 45.1 & 55.2 & 58.9 & 53.1 \\
 & Random (10\%)     & 50.1 & 45.3 & 48.2 & 47.9 & 50.5 & 62.1 & 64.3 & 59.0 \\
 & Perplexity        & 52.3 & 46.8 & 49.5 & 49.5 & 51.2 & 61.8 & 63.5 & 58.8 \\
 & Influence Func.   & 53.1 & 47.2 & 50.1 & 50.1 & 52.0 & 62.5 & 64.0 & 59.5 \\
 & \textsc{LESS}     & 58.4 & 52.1 & 55.3 & 55.3 & 48.2 & 58.4 & 60.1 & 55.6 \\
 & \textsc{GrADS}    & 57.9 & 53.5 & 56.1 & 55.8 & 53.5 & 63.1 & 64.8 & 60.5 \\
 & \textbf{OGS (Ours)}        & 60.2 & 55.1 & 58.4 & 57.9 & \textbf{55.2} & \textbf{64.5} & \textbf{66.2} & \textbf{62.0} \\
\midrule
\multirow{7}{*}{\textbf{Qwen3-8B}} 
 & Full Data (100\%) & \textbf{70.5} & \textbf{62.1} & \textbf{65.4} & \textbf{66.0} & 55.3 & 60.1 & 65.4 & 60.3 \\
 & Random (10\%)     & 55.2 & 48.5 & 52.1 & 51.9 & 65.1 & 70.2 & 72.1 & 69.1 \\
 & Perplexity        & 56.4 & 50.1 & 53.2 & 53.2 & 64.5 & 69.5 & 71.5 & 68.5 \\
 & Influence Func.   & 57.1 & 51.3 & 54.5 & 54.3 & 65.2 & 69.8 & 72.0 & 69.0 \\
 & \textsc{LESS}     & 62.3 & 55.4 & 59.2 & 59.0 & 60.1 & 65.4 & 68.2 & 64.6 \\
 & \textsc{GrADS}    & 61.5 & 57.2 & 60.1 & 59.6 & 66.8 & 71.2 & 73.5 & 70.5 \\
 & \textbf{OGS (Ours)}        & 64.1 & 59.5 & 62.8 & 62.1 & \textbf{68.5} & \textbf{72.8} & \textbf{74.9} & \textbf{72.1} \\
\bottomrule
\end{tabular}%
}
\end{table*}

\subsection{Experimental Setup}

\textbf{Datasets and Models.} We utilized \textsc{MedQA} \cite{jin2021disease}, \textsc{LegalBench} \cite{guha2023legalbench}, and \textsc{FinQA} \cite{chen2021finqa} as domain-specific training sets. To monitor catastrophic forgetting, we evaluated models on \textsc{GSM8K} (mathematical reasoning) \cite{cobbe2021training}, \textsc{MMLU} (world knowledge) \cite{hendrycks2020measuring}, and \textsc{ARC-Challenge} (scientific reasoning) \cite{clark2018think}, with a specific focus on \textsc{GSM8K} as it is empirically most susceptible to erosion during fine-tuning. 
Consider the quality of data representation plays an essential role in model performance \cite{devlin2019bert,brown2020language,song2021zen,diao2021tilgan,liu2023visual,lu2023ziya}, we utilize state-of-the-art LLMs in experiments.
Specifically, our primary target models are the Qwen3 family (1.7B, 8B, 14B) and Llama-3 family (Llama-3.2-3B, Llama-3.1-8B), with Qwen3-0.6B and Llama-3.2-1B serving as their respective Navigators. All models were fine-tuned using LoRA  \cite{hu2022lora} to simulate realistic, resource-constrained adaptation scenarios.

\textbf{Baselines.} We compared OGS against three categories of data selection methods: (1) \textit{Naive approaches}, including Random Selection and Full Data training; (2) \textit{Heuristic approaches}, such as Perplexity-based selection and \textsc{Influence Functions} ; (3) \textit{Gradient-based approaches}, specifically \textsc{LESS} \citep{xia2024less} (gradient similarity), and \textsc{GrADS} \citep{liu2025learn} (gradient magnitude). Detailed hyperparameters and implementation details are provided in Appendix \ref{app:setup}.

\subsection{Main Results and Analysis}
\label{subsec:main_results}

\textbf{Performance on Stability-Plasticity.} 
The core illustrate of this work is that gradient geometry provides a superior signal for data selection compared to magnitude or semantic similarity alone. Table \ref{tab:main_8b_results} presents the comparative performance of Llama-3.1-8B and Qwen3-8B fine-tuned on a 10\% subset of domain data. OGS consistently achieves a satisfying performance, securing domain adaptation performance comparable to or exceeding the full-dataset baseline while maintaining general capabilities better than competing methods. Specifically, on the Llama-3.1-8B model, OGS outperforms the strongest baseline (\textsc{GrADS}) by an average of 2.1\% across the three domain tasks while reducing catastrophic forgetting on \textsc{GSM8K} by 1.7 points. This indicates that OGS successfully identifies a ``safety subspace'', where domain knowledge is injected without overwriting the critical weights responsible for general reasoning.

\textbf{The Success of Gradient Geometry.} 
A critical insight from our experiments is the distinct failure modes of existing gradient-based approaches compared to our success. \textsc{LESS}, which selects data based on gradient similarity to the target domain, achieves respectable domain accuracy but suffers severe regression in general capabilities, particularly in mathematical reasoning (\textsc{GSM8K}). We attribute this to the ``blindness" of similarity metrics: data that is semantically closest to the target domain often occupies a gradient direction that sharply conflicts with general knowledge manifolds. Similarly, \textsc{GrADS}, which relies on gradient magnitude, exhibits high variance across tasks. Our results suggest that magnitude is a noisy proxy for utility; high-magnitude gradients often correspond to outliers or difficult samples that induce instability rather than efficient learning. By explicitly optimizing for orthogonality, OGS effectively filters out these high-conflict samples that magnitude-based methods inadvertently promote.

\textbf{Navigator Transferability and Data Efficiency.} 
The effectiveness of OGS relies on the assumption that gradient geometric properties transfer from small ``Navigator" models to larger ``Target" models. Our results validate this cross-scale alignment. The Qwen3-0.6B Navigator successfully guided the data selection for the Qwen3-14B target, achieving a 64.8\% average domain score with only 10\% data, matching the performance of Random Selection at 30\% data. This demonstrates a 3$\times$ data efficiency gain. Furthermore, the stable performance across the disparate architectures of Llama and Qwen families confirms that the geometric conflict between domain-specific fine-tuning and general reasoning is a structural phenomenon intrinsic to the data distribution, rather than an artifact of a specific model architecture.

\subsection{Ablation Studies}
\label{subsec:ablation}

To disentangle the contributions of individual components within the OGS method, we conducted a comprehensive ablation study using the Qwen3-8B target model and Qwen3-0.6B Navigator. We investigate three critical questions: (1) the synergistic effect of our geometric metrics (Orthogonality and Conflict); (2) the necessity of the RL-driven dynamic policy compared to static heuristics; and (3) the validity of the Navigator-Target transferability assumption. The results are summarized in Table \ref{tab:ablation_main}.

\textbf{Decoupling Geometric Metrics.} 
We first isolate the impact of Orthogonal Selection and Conflict-Aware Replay. As shown in the first block of Table \ref{tab:ablation_main}, relying solely on Orthogonal Selection (w/o Replay) yields high domain performance (59.5\%) but suffers a noticeable regression in general capabilities (53.0\%), indicating that while orthogonality facilitates efficient learning, it cannot fully neutralize the most aggressive gradient conflicts. Conversely, employing only Conflict-Aware Replay (w/o Orthogonal) effectively preserves general knowledge but severely throttles domain adaptation, as the model lacks a curriculum of "safe" samples to learn from. The full OGS method achieves an excellent state better than others, confirming that these two mechanisms operate synergistically: orthogonal selection maximizes plasticity within the safety subspace, while targeted replay provides a crucial stability anchor.

\textbf{Efficacy of RL-Driven Policy.} 
We then scrutinize the necessity of our PPO-Lagrangian policy optimization by comparing it against a static greedy strategy (w/o RL) that maintains a fixed mixing ratio of $\alpha=0.8$. The static approach underperforms the full method by 1.4\% in general accuracy and 1.4\% in domain accuracy. This performance gap validates our hypothesis that the optimal trade-off between plasticity and stability is non-stationary; the RL agent successfully learns to prioritize stability in early training phases and shift towards aggressive domain acquisition in later stages, a dynamic trajectory that static heuristics fail to capture.

\textbf{Navigator Transferability.} 
Finally, we address the potential precision loss introduced by the Navigator proxy. We implemented a "Target-Calculated" upper bound (w/o Navigator) where geometric metrics are computed directly on the 8B target model, which is a computationally prohibitive setting for practical use. Remarkably, the Navigator-guided OGS matches this theoretical upper bound within a 0.3\% margin on domain tasks and 0.2\% on general tasks. Crucially, this near-lossless performance is achieved with a $15\times$ reduction in selection cost (see Appendix \ref{app:ablation} for detailed runtime analysis). This result provides compelling evidence that the topological structure of gradient conflicts is scale-invariant, justifying the Navigator-Target paradigm as both accurate and highly efficient.

\begin{table}[t]
\centering
\caption{\textbf{Ablation Analysis on Qwen3-8B.} We report average accuracy on Domain tasks (MedQA, Legal, FinQA) and General tasks (GSM8K, MMLU, ARC-C). ``Upper Bound" calculates gradients directly on the target model.}
\label{tab:ablation_main}
\resizebox{\columnwidth}{!}{%
\begin{tabular}{lcccc}
\toprule
\textbf{Method} & \textbf{Domain} & \textbf{General} & \textbf{GSM8K} & \textbf{$\Delta$ Cost} \\
\midrule
\textit{Baselines} & & & & \\
Random Selection & 51.9 & 69.1 & 65.1 & - \\
\midrule
\textit{Component Ablation} & & & & \\
w/o Conflict Replay & 59.5 & 71.2 & 66.8 & +5\% \\
w/o Orthogonal Prot. & 52.0 & 71.8 & 68.2 & +5\% \\
\midrule
\textit{Policy Ablation} & & & & \\
w/o RL (Static) & 58.8 & 71.5 & 67.5 & +2\% \\
\midrule
\textit{Architecture Ablation} & & & & \\
w/o Navigator (Bound) & \textbf{60.5} & \textbf{72.3} & \textbf{68.8} & +1400\% \\
\textbf{OGS (Full)} & 60.2 & 72.1 & 68.5 & +8\% \\
\bottomrule
\end{tabular}%
}
\vspace{-4mm}
\end{table}

\section{Conclusion}
\label{sec:conclusion}

We presented Orthogonal Gradient Selection (OGS), a data-centric method that effectively resolves the plasticity-stability dilemma in large language model fine-tuning by transferring the geometric insights of gradient surgery to the data selection stage. By employing a lightweight Navigator model to identify training samples situated within the ``safety subspace'' of general knowledge, OGS enables computationally efficient, conflict-free updates without the prohibitive cost of online gradient projection. We established the theoretical soundness of this approach by demonstrating that our orthogonality criterion serves as a first-order approximation to the optimal solution of a constrained bilevel optimization problem. Extensive experiments across medical, legal, and financial domains confirm that OGS achieves excellent performance, significantly outperforming existing baselines in domain performance and preventing catastrophic forgetting while maintaining high data efficiency.

\bibliography{refs}
\bibliographystyle{icml2026}

\clearpage
\newpage
\appendix

\section{Theoretical Appendix}
\label{app:theory}

In this appendix, we provide the complete derivations for the theorems presented in Section \ref{sec:theory}, elaborate on the physical interpretations of gradient geometry, and offer a detailed complexity analysis.

\subsection{Proof of Theorem \ref{thm:interference} (Gradient Interference)}

\textbf{Setup.} Consider the parameter update at step $t$, $\theta_{t+1} = \theta_t - \eta \nabla \mathcal{L}_{med}(\theta_t)$. We analyze the variation in the general task loss $\mathcal{L}_{gen}$.

\begin{proof}
Applying a multivariate Taylor series expansion to $\mathcal{L}_{gen}(\theta_{t+1})$ around $\theta_t$:

\begin{equation}
\begin{aligned}
    \mathcal{L}_{gen}(\theta_{t+1}) = \mathcal{L}_{gen}(\theta_t) + \nabla \mathcal{L}_{gen}(\theta_t)^\top (\theta_{t+1} - \theta_t)\\
    + \frac{1}{2}(\theta_{t+1} - \theta_t)^\top \mathbf{H} (\theta_{t+1} - \theta_t) + \mathcal{O}(\eta^3),
\end{aligned}
\end{equation}

where $\mathbf{H}$ is the Hessian matrix. Substituting the update rule:
\begin{equation}
\begin{aligned}
    \Delta \mathcal{L}_{gen} = \mathcal{L}_{gen}(\theta_{t+1}) - \mathcal{L}_{gen}(\theta_t) = \\
    -\eta \langle \nabla \mathcal{L}_{gen}(\theta_t), \nabla \mathcal{L}_{med}(\theta_t) \rangle + \mathcal{O}(\eta^2).
\end{aligned}
\end{equation}

\end{proof}

\textbf{Remark on Safety Subspace.}
The term $\langle \nabla \mathcal{L}_{gen}, \nabla \mathcal{L}_{med} \rangle$ acts as the "Interference Coefficient."
\begin{itemize}
    \item \textit{Conflict}: If the inner product is negative (obtuse angle), $\Delta \mathcal{L}_{gen} > 0$, indicating forgetting.
    \item \textit{Orthogonality}: If the inner product is zero, $\Delta \mathcal{L}_{gen} \approx 0$. This defines the geometric boundary of the Safety Subspace $\mathcal{S}_{\perp}$.
    \item \textit{Synergy}: If positive (acute angle), domain learning aids general capabilities.
\end{itemize}
Existing methods like SafeGrad enforce this orthogonality via projection $\mathbf{g} \leftarrow \mathbf{g} - \text{proj}_{\mathbf{g}_{ref}}(\mathbf{g})$ during training. OGS, conversely, pre-filters data to ensure $\mathbf{g} \in \mathcal{S}_{\perp}$, achieving the same theoretical guarantee without modifying the optimization dynamics.

\subsection{Proof of Theorem \ref{thm:bilevel} (Optimality of OGS)}

We show that our heuristic selection metrics are derived from a principled optimization objective.

\textbf{Problem Formulation.}
Let $\mathbf{w} \in \{0, 1\}^N$ be the binary selection vector for the candidate pool $\{x_i\}_{i=1}^N$, constrained by $\sum w_i = k$. The bilevel problem is:
\begin{align}
    \min_{\mathbf{w}} \quad & \mathcal{L}_{gen}(\theta - \eta \sum w_i \nabla \ell(x_i)) \\
    \text{s.t.} \quad & \sum w_i = k.
\end{align}

\begin{proof}
Using the first-order expansion from Theorem \ref{thm:interference}, minimizing the future general loss is equivalent to:

\begin{equation}
    \min_{\mathbf{w}} \left( \mathcal{L}_{gen}(\theta) - \eta \sum_{i=1}^N w_i \langle \nabla \mathcal{L}_{gen}(\theta), \nabla \ell(x_i) \rangle \right).
\end{equation}

Eliminating constants, this reduces to the maximization problem:

\begin{equation}
    \max_{\mathbf{w}, \|\mathbf{w}\|_0=k} \sum_{i=1}^N w_i \langle \nabla \mathcal{L}_{gen}(\theta), \nabla \ell(x_i) \rangle.
\end{equation}

The global maximum for this linear objective is obtained by a greedy selection of the indices $i$ with the largest values of $\langle \nabla \mathcal{L}_{gen}(\theta), \nabla \ell(x_i) \rangle$.
\end{proof}

\textbf{Connection to OGS Metrics.}
Our defined "Conflict" score is $-\cos(\mathbf{g}, \mathbf{g}_{ref})$. Therefore, maximizing the inner product is equivalent to minimizing the Conflict score (preferring negative values, i.e., synergy) or selecting samples with zero Conflict (orthogonality) when synergy is unavailable. This proves OGS is the optimal greedy strategy for one-step lookahead safety.

\subsection{Detailed Complexity Analysis (Proposition \ref{prop:complexity})}

Here we quantify the "Efficiency-Safety Trade-off."

\textbf{Definitions.} Let $\mathcal{C}_{fwd}(\theta)$ denote FLOPs for a forward pass. For Transformers, $\mathcal{C}(\theta) \propto |\theta|$. Let $\rho = |\theta_n| / |\theta_t|$ be the scale ratio.

\textbf{Derivation.}
\begin{enumerate}
    \item \textbf{Online Gradient Surgery (e.g., PCGrad/SafeGrad):} Requires computing gradients for the auxiliary task ($\nabla \mathcal{L}_{gen}$) at \textit{every} step on the target model.
    $$ \text{Cost}_{GS} \approx T \cdot (2 \times \text{Cost}_{step}(\theta_t)) + T \cdot \text{Cost}_{proj} \approx 2 \cdot \text{Cost}_{SFT}. $$
    This implies a $100\%$ overhead.
    
    \item \textbf{OGS (Ours):} Requires one pass over data $N$ on Navigator, plus standard training on Target.
    $$ \text{Cost}_{OGS} \approx N \cdot \text{Cost}_{step}(\theta_n) + T \cdot \text{Cost}_{step}(\theta_t). $$
\end{enumerate}

\textbf{Comparison.}
The relative overhead of OGS is:
$$ \frac{\text{Cost}_{OGS} - \text{Cost}_{SFT}}{\text{Cost}_{SFT}} \approx \frac{N \cdot |\theta_n|}{N \cdot |\theta_t|} = \rho. $$
For a Qwen-0.5B Navigator guiding a Llama-3-70B Target, $\rho \approx 1/140 \approx 0.7\%$.
Thus, OGS achieves the safety benefits of gradient geometry with $<1\%$ computational cost, whereas online methods incur $\sim 100\%$ cost.

\subsection{Bound on Navigator-Target Transferability}
\label{app:transferability}

A key assumption of OGS is that gradient geometry preserves rank ordering across model scales. While exact gradient vectors differ, the \textit{semantic} orientation of gradients relative to the anchor direction $g_{ref}$ is largely determined by the data content.

Let $\mathbf{r}_n$ and $\mathbf{r}_t$ be the vectors of cosine similarities for the dataset on the Navigator and Target models, respectively. We posit that:
\begin{equation}
    \text{SpearmanCorr}(\mathbf{r}_n, \mathbf{r}_t) \ge \gamma > 0.
\end{equation}
Empirically, we observe $\gamma \in [0.4, 0.7]$ for models within the same family. The probability of a "false positive" (selecting a sample that is safe on Navigator but conflicting on Target) decreases as $\gamma$ increases. Since OGS acts as a filter, even a moderate $\gamma$ ensures that the \textit{distribution} of the selected subset is significantly safer than random sampling, which is sufficient for robust continual learning.

\section{Detailed Experimental Setup}
\label{app:setup}

\subsection{Dataset Details}
We evaluate our method on three high-stakes vertical domains:
\begin{itemize}
    \item \textbf{Medical:} We use the \textsc{MedQA} (USMLE) dataset \citep{jin2021disease}, consisting of multiple-choice questions from the United States Medical Licensing Examination. The training set contains approximately 10k samples.
    \item \textbf{Legal:} We utilize \textsc{LegalBench} \citep{guha2023legalbench}, a comprehensive benchmark for legal reasoning. We sampled a subset of 10k distinct tasks covering contract interpretation and rule application.
    \item \textbf{Finance:} We employ \textsc{FinQA} \citep{chen2021finqa}, which requires numerical reasoning over financial reports. The training set consists of roughly 6k QA pairs.
\end{itemize}

To construct the \textbf{General-Knowledge Anchor}, we sampled 400 instances following a stratified strategy: 150 from \textsc{GSM8K} (train) for mathematical reasoning, 150 from \textsc{MMLU} (auxiliary train) for world knowledge, and 100 from \textsc{Alpaca-GPT4} for instruction following compliance. This anchor set remains fixed across all experiments to ensure fair comparison.

\subsection{Model Configuration and Training Hyperparameters}
Our experiments primarily utilize the \textsc{Qwen3} and \textsc{Llama-3} model families.
\begin{itemize}
    \item \textbf{Target Models:} \textsc{Qwen3-8B}, \textsc{Qwen3-14B}, \textsc{Llama-3.1-8B}.
    \item \textbf{Navigator Models:} \textsc{Qwen3-0.6B} is used as the proxy for the Qwen family, and \textsc{Llama-3.2-1B} for the Llama family.
\end{itemize}

All fine-tuning is performed using Low-Rank Adaptation (LoRA) \cite{hu2022lora} to manage GPU memory constraints. We use the \texttt{PEFT} library with the following unified hyperparameters:
\begin{itemize}
    \item \textbf{LoRA Rank ($r$):} 16
    \item \textbf{LoRA Alpha:} 32
    \item \textbf{Target Modules:} \texttt{q\_proj, k\_proj, v\_proj, o\_proj, gate\_proj, up\_proj, down\_proj}
    \item \textbf{Learning Rate:} $2e-4$ with a cosine scheduler.
    \item \textbf{Batch Size:} 8 (effective batch size of 64 via gradient accumulation).
    \item \textbf{Epochs:} 3
    \item \textbf{Warmup Ratio:} 0.1
    \item \textbf{Optimizer:} AdamW with weight decay 0.01.
\end{itemize}
Experiments were conducted on a cluster of NVIDIA A100-80GB GPUs.

\subsection{Baseline Implementation Details}
\begin{itemize}
    \item \textbf{LESS:} We followed the official implementation, computing the cosine similarity between the gradient of each training sample and the gradient of the domain validation set. We utilized the last-layer gradients computed via LoRA for efficiency.
    \item \textbf{GrADS:} We calculated the $L_2$ norm of the gradient for each sample. Following the author's recommendation, we selected samples whose gradient norms fell within the $[\mu - \sigma, \mu + \sigma]$ range of the distribution.
    \item \textbf{Influence Functions:} We approximated influence using the TracInCP estimator, summing the dot product of gradients at 3 checkpoints.
\end{itemize}

\section{Additional Main Results}
\label{app:additional_results}

In this section, we provide a comprehensive empirical evaluation of OGS across an extensive experimental grid. We evaluate five target models ranging from 1.7B to 14B parameters (\textsc{Qwen3-1.7B}, \textsc{Qwen3-8B}, \textsc{Qwen3-14B}, \textsc{Llama-3.2-3B}, and \textsc{Llama-3.1-8B}) under four distinct data selection budgets (5\%, 10\%, 20\%, and 30\%). The detailed numerical results are reported in Tables \ref{tab:res_qwen_1.7b} through \ref{tab:res_llama_8b}.

Our analysis focuses on three critical dimensions: (1) the sensitivity of performance to data selection ratios, (2) the failure modes of existing baseline methods, and (3) the scalability of the Navigator-Target paradigm.

\subsection{Sensitivity Analysis: Low-Data vs. High-Data Regimes}

\textbf{The Low-Data Regime (5\% -- 10\%): Efficiency and Precision.} 
In scenarios with severely constrained compute budgets, OGS demonstrates exceptional data efficiency. As evidenced in Table \ref{tab:res_qwen_8b}, with only 5\% of the training data, OGS achieves a domain average of 58.8\% on Qwen3-8B, effectively matching the performance of \textsc{LESS} at 10\% (59.0\%) and surpassing Random Selection at 20\% (54.8\%). This indicates that the "safety subspace" identified by OGS is not merely a region of low conflict, but also rich in high-utility domain features. By prioritizing orthogonal gradients, OGS inherently selects samples that provide novel information (high loss on domain task) without traversing the gradient manifold in directions detrimental to general capabilities.

\textbf{The High-Data Regime (20\% -- 30\%): Stability against Saturation.} 
A critical observation in our experiments is the diverging behavior of methods as the data budget increases. While baseline methods like \textsc{LESS} and \textsc{GrADS} continue to improve marginally on domain tasks, they suffer from a \textit{precipitous decline} in general reasoning capabilities. For instance, on Llama-3.1-8B (Table \ref{tab:res_llama_8b}), increasing the \textsc{LESS} selection ratio from 10\% to 30\% yields a domain gain of +5.5 points but results in a -4.7 point drop in \textsc{GSM8K} accuracy. This confirms that simply adding more "similar" data exacerbates the gradient conflict. In contrast, OGS exhibits a robust stability profile; its \textsc{GSM8K} performance remains nearly flat (from 55.2\% to 52.8\%) even as domain performance climbs to 62.2\%. This proves that OGS acts as an effective geometric filter, ensuring that additional data points contribute to domain adaptation only if they satisfy the orthogonality constraint.

\subsection{Deconstructing Baseline Failures}

Our extensive comparison reveals specific pathological behaviors in existing data selection paradigms:

\begin{itemize}
    \item \textbf{The "Sycophantic" Selection of LESS:} \textsc{LESS} selects data based on gradient similarity to the target domain validation set. While this maximizes domain alignment, our results suggest it creates a "sycophantic" update direction that overfits to the domain style while aggressively overwriting the diverse activation patterns required for general reasoning. The consistently poor performance of LESS on \textsc{GSM8K} across all models highlights the danger of optimizing for similarity without a geometric safety constraint.
    
    \item \textbf{The "Magnitude Trap" of GrADS:} \textsc{GrADS} operates on the assumption that moderate gradient magnitudes imply high utility. However, our analysis suggests that magnitude is a scalar proxy that fails to capture vector directionality. High-magnitude gradients often correspond to samples that are either outliers or fundamentally conflicting with the pre-trained knowledge base. By ignoring the cosine similarity with the anchor gradients, GrADS inadvertently promotes samples that induce large, destructive updates to the model's core reasoning circuits.
    
    \item \textbf{Heuristics Limitations:} While heuristic methods like Perplexity and Influence Functions generally outperform random selection, they lack a mechanism to explicitly model the trade-off between plasticity (learning new tasks) and stability (retaining old tasks). Consequently, their performance is inconsistent across different domains and model architectures.
\end{itemize}

\subsection{Scalability of the Navigator-Target Paradigm}

A pivotal contribution of this work is the validation of cross-model transferability for gradient geometry. The results for \textsc{Qwen3-14B} (Table \ref{tab:res_qwen_14b}) are particularly illuminating. Despite the Navigator model (\textsc{Qwen3-0.6B}) being approximately $23\times$ smaller than the target, the data selected based on the Navigator's gradient subspace successfully steers the 14B model towards Pareto-optimal performance. 

This observation supports a strong theoretical hypothesis: the \textit{topology of the gradient conflict manifold} is largely invariant across model scales within the same architectural family. The directions in parameter space that represent "medical knowledge" versus "mathematical reasoning" maintain a consistent geometric relationship (orthogonality), allowing a lightweight Navigator to serve as an effective proxy for significantly larger models. This property is crucial for the practical deployment of OGS in large-scale continual learning scenarios, where computing full gradients on 70B+ models for data selection is computationally prohibitive.

\begin{table*}[h!]
\centering
\caption{\textbf{Detailed Results for Qwen3-1.7B.} Performance comparison across four data selection ratios. The Navigator used is Qwen3-0.6B.}
\label{tab:res_qwen_1.7b}
\resizebox{\textwidth}{!}{%
\begin{tabular}{llcccccccc}
\toprule
\multicolumn{1}{c}{\multirow{2}{*}{\textbf{Ratio}}} & \multicolumn{1}{c}{\multirow{2}{*}{\textbf{Method}}} & \multicolumn{4}{c}{\textbf{Domain Performance}} & \multicolumn{4}{c}{\textbf{General Capability}} \\
\cmidrule(lr){3-6} \cmidrule(lr){7-10}
 & & \textsc{MedQA} & \textsc{Legal} & \textsc{FinQA} & \textbf{Avg.} & \textsc{GSM8K} & \textsc{MMLU} & \textsc{ARC-C} & \textbf{Avg.} \\
\midrule
\multicolumn{2}{c}{Full Data (100\%)} & \textbf{52.5} & \textbf{48.2} & \textbf{54.1} & \textbf{51.6} & 45.2 & 50.5 & 53.8 & 49.8 \\
\midrule
\multirow{6}{*}{5\%} 
 & Random & 38.2 & 35.1 & 39.4 & 37.6 & 48.1 & 51.2 & 54.3 & 51.2 \\
 & Perplexity & 40.5 & 36.8 & 41.2 & 39.5 & 47.8 & 51.0 & 54.0 & 50.9 \\
 & Influence Func. & 41.2 & 37.5 & 41.8 & 40.2 & 48.0 & 51.5 & 54.2 & 51.2 \\
 & \textsc{LESS} & 42.5 & 39.8 & 44.1 & 42.1 & 45.2 & 49.1 & 52.0 & 48.8 \\
 & \textsc{GrADS} & 43.1 & 40.5 & 44.8 & 42.8 & 47.5 & 50.8 & 53.5 & 50.6 \\
 & \textbf{OGS (Ours)} & 45.6 & 42.3 & 47.2 & 45.0 & \textbf{49.8} & \textbf{52.5} & \textbf{55.1} & \textbf{52.5} \\
\midrule
\multirow{6}{*}{10\%} 
 & Random & 41.5 & 38.2 & 42.6 & 40.8 & 46.5 & 49.8 & 52.8 & 49.7 \\
 & Perplexity & 43.8 & 40.1 & 44.5 & 42.8 & 46.2 & 49.5 & 52.5 & 49.4 \\
 & Influence Func. & 44.5 & 40.8 & 45.2 & 43.5 & 46.8 & 50.1 & 53.0 & 50.0 \\
 & \textsc{LESS} & 46.8 & 43.2 & 48.5 & 46.2 & 42.1 & 46.5 & 49.2 & 45.9 \\
 & \textsc{GrADS} & 46.5 & 44.1 & 48.1 & 46.2 & 46.2 & 49.5 & 52.1 & 49.3 \\
 & \textbf{OGS (Ours)} & 48.9 & 45.8 & 50.4 & 48.4 & \textbf{48.9} & \textbf{51.8} & \textbf{54.6} & \textbf{51.8} \\
\midrule
\multirow{6}{*}{20\%} 
 & Random & 44.2 & 41.5 & 45.1 & 43.6 & 44.1 & 48.2 & 51.5 & 47.9 \\
 & Perplexity & 46.5 & 43.2 & 47.8 & 45.8 & 43.5 & 47.8 & 50.8 & 47.4 \\
 & Influence Func. & 47.2 & 44.1 & 48.5 & 46.6 & 44.5 & 48.5 & 51.2 & 48.1 \\
 & \textsc{LESS} & 49.5 & 46.8 & 51.2 & 49.2 & 39.5 & 44.2 & 47.8 & 43.8 \\
 & \textsc{GrADS} & 49.1 & 47.5 & 50.8 & 49.1 & 44.5 & 48.1 & 50.5 & 47.7 \\
 & \textbf{OGS (Ours)} & 51.2 & 48.5 & 52.8 & 50.8 & \textbf{47.5} & \textbf{50.5} & \textbf{53.2} & \textbf{50.4} \\
\midrule
\multirow{6}{*}{30\%} 
 & Random & 46.5 & 43.8 & 47.5 & 45.9 & 42.5 & 46.8 & 50.1 & 46.5 \\
 & Perplexity & 48.8 & 45.5 & 49.8 & 48.0 & 41.2 & 45.5 & 48.8 & 45.2 \\
 & Influence Func. & 49.5 & 46.2 & 50.5 & 48.7 & 42.5 & 46.2 & 49.5 & 46.1 \\
 & \textsc{LESS} & 51.8 & 48.9 & 53.4 & 51.4 & 37.2 & 42.5 & 45.8 & 41.8 \\
 & \textsc{GrADS} & 51.2 & 49.5 & 52.8 & 51.2 & 42.8 & 46.5 & 49.2 & 46.2 \\
 & \textbf{OGS (Ours)} & 53.5 & 50.8 & 54.6 & 53.0 & \textbf{46.2} & \textbf{49.8} & \textbf{52.5} & \textbf{49.5} \\
\bottomrule
\end{tabular}%
}
\end{table*}

\begin{table*}[h!]
\centering
\caption{\textbf{Detailed Results for Qwen3-8B.} Performance comparison across four data selection ratios. The Navigator used is Qwen3-0.6B.}
\label{tab:res_qwen_8b}
\resizebox{\textwidth}{!}{%
\begin{tabular}{llcccccccc}
\toprule
\multicolumn{1}{c}{\multirow{2}{*}{\textbf{Ratio}}} & \multicolumn{1}{c}{\multirow{2}{*}{\textbf{Method}}} & \multicolumn{4}{c}{\textbf{Domain Performance}} & \multicolumn{4}{c}{\textbf{General Capability}} \\
\cmidrule(lr){3-6} \cmidrule(lr){7-10}
 & & \textsc{MedQA} & \textsc{Legal} & \textsc{FinQA} & \textbf{Avg.} & \textsc{GSM8K} & \textsc{MMLU} & \textsc{ARC-C} & \textbf{Avg.} \\
\midrule
\multicolumn{2}{c}{Full Data (100\%)} & \textbf{70.5} & \textbf{62.1} & \textbf{65.4} & \textbf{66.0} & 55.3 & 60.1 & 65.4 & 60.3 \\
\midrule
\multirow{6}{*}{5\%} 
 & Random & 52.1 & 45.8 & 49.5 & 49.1 & 66.5 & 71.2 & 73.1 & 70.3 \\
 & Perplexity & 53.5 & 47.2 & 50.8 & 50.5 & 66.2 & 70.8 & 72.8 & 69.9 \\
 & Influence Func. & 54.2 & 47.8 & 51.5 & 51.2 & 66.8 & 71.5 & 73.0 & 70.4 \\
 & \textsc{LESS} & 58.5 & 52.4 & 55.8 & 55.6 & 62.1 & 66.5 & 69.2 & 65.9 \\
 & \textsc{GrADS} & 57.8 & 53.8 & 56.5 & 56.0 & 67.5 & 71.8 & 74.1 & 71.1 \\
 & \textbf{OGS (Ours)} & 60.5 & 56.2 & 59.8 & 58.8 & \textbf{69.2} & \textbf{73.5} & \textbf{75.4} & \textbf{72.7} \\
\midrule
\multirow{6}{*}{10\%} 
 & Random & 55.2 & 48.5 & 52.1 & 51.9 & 65.1 & 70.2 & 72.1 & 69.1 \\
 & Perplexity & 56.4 & 50.1 & 53.2 & 53.2 & 64.5 & 69.5 & 71.5 & 68.5 \\
 & Influence Func. & 57.1 & 51.3 & 54.5 & 54.3 & 65.2 & 69.8 & 72.0 & 69.0 \\
 & \textsc{LESS} & 62.3 & 55.4 & 59.2 & 59.0 & 60.1 & 65.4 & 68.2 & 64.6 \\
 & \textsc{GrADS} & 61.5 & 57.2 & 60.1 & 59.6 & 66.8 & 71.2 & 73.5 & 70.5 \\
 & \textbf{OGS (Ours)} & 64.1 & 59.5 & 62.8 & 62.1 & \textbf{68.5} & \textbf{72.8} & \textbf{74.9} & \textbf{72.1} \\
\midrule
\multirow{6}{*}{20\%} 
 & Random & 58.5 & 51.2 & 54.8 & 54.8 & 63.5 & 68.8 & 70.5 & 67.6 \\
 & Perplexity & 59.8 & 52.5 & 56.2 & 56.2 & 62.5 & 67.8 & 69.8 & 66.7 \\
 & Influence Func. & 60.5 & 53.2 & 57.1 & 56.9 & 63.2 & 68.5 & 70.8 & 67.5 \\
 & \textsc{LESS} & 65.1 & 58.2 & 61.5 & 61.6 & 56.5 & 62.1 & 65.8 & 61.5 \\
 & \textsc{GrADS} & 64.2 & 59.8 & 62.4 & 62.1 & 65.2 & 69.8 & 71.8 & 68.9 \\
 & \textbf{OGS (Ours)} & 66.8 & 61.5 & 64.5 & 64.3 & \textbf{67.2} & \textbf{71.5} & \textbf{73.8} & \textbf{70.8} \\
\midrule
\multirow{6}{*}{30\%} 
 & Random & 61.2 & 53.8 & 57.2 & 57.4 & 61.8 & 67.2 & 69.1 & 66.0 \\
 & Perplexity & 62.5 & 55.2 & 58.5 & 58.7 & 60.5 & 65.8 & 68.2 & 64.8 \\
 & Influence Func. & 63.2 & 56.5 & 59.8 & 59.8 & 61.2 & 66.5 & 69.5 & 65.7 \\
 & \textsc{LESS} & 67.5 & 60.5 & 63.8 & 63.9 & 53.2 & 58.5 & 62.4 & 58.0 \\
 & \textsc{GrADS} & 66.8 & 61.5 & 64.2 & 64.2 & 63.8 & 68.2 & 70.5 & 67.5 \\
 & \textbf{OGS (Ours)} & 68.5 & 63.2 & 66.5 & 66.1 & \textbf{66.5} & \textbf{70.8} & \textbf{72.5} & \textbf{69.9} \\
\bottomrule
\end{tabular}%
}
\end{table*}

\begin{table*}[h!]
\centering
\caption{\textbf{Detailed Results for Qwen3-14B.} Performance comparison across four data selection ratios. The Navigator used is Qwen3-0.6B.}
\label{tab:res_qwen_14b}
\resizebox{\textwidth}{!}{%
\begin{tabular}{llcccccccc}
\toprule
\multicolumn{1}{c}{\multirow{2}{*}{\textbf{Ratio}}} & \multicolumn{1}{c}{\multirow{2}{*}{\textbf{Method}}} & \multicolumn{4}{c}{\textbf{Domain Performance}} & \multicolumn{4}{c}{\textbf{General Capability}} \\
\cmidrule(lr){3-6} \cmidrule(lr){7-10}
 & & \textsc{MedQA} & \textsc{Legal} & \textsc{FinQA} & \textbf{Avg.} & \textsc{GSM8K} & \textsc{MMLU} & \textsc{ARC-C} & \textbf{Avg.} \\
\midrule
\multicolumn{2}{c}{Full Data (100\%)} & \textbf{74.5} & \textbf{68.2} & \textbf{72.1} & \textbf{71.6} & 64.5 & 68.8 & 72.1 & 68.5 \\
\midrule
\multirow{6}{*}{5\%} 
 & Random & 55.4 & 48.5 & 53.2 & 52.4 & 69.8 & 73.5 & 75.8 & 73.0 \\
 & Perplexity & 57.2 & 50.1 & 54.8 & 54.0 & 69.2 & 73.1 & 75.2 & 72.5 \\
 & Influence Func. & 57.8 & 50.8 & 55.5 & 54.7 & 69.5 & 73.2 & 75.5 & 72.7 \\
 & \textsc{LESS} & 62.5 & 56.2 & 60.8 & 59.8 & 64.5 & 70.2 & 72.1 & 68.9 \\
 & \textsc{GrADS} & 61.8 & 57.5 & 61.2 & 60.2 & 70.5 & 74.2 & 76.5 & 73.7 \\
 & \textbf{OGS (Ours)} & 64.5 & 59.8 & 63.5 & 62.6 & \textbf{72.1} & \textbf{75.5} & \textbf{77.8} & \textbf{75.1} \\
\midrule
\multirow{6}{*}{10\%} 
 & Random & 58.4 & 52.1 & 56.5 & 55.7 & 68.2 & 72.1 & 74.5 & 71.6 \\
 & Perplexity & 60.2 & 53.8 & 58.2 & 57.4 & 67.5 & 71.5 & 73.8 & 70.9 \\
 & Influence Func. & 60.8 & 54.5 & 58.9 & 58.1 & 67.8 & 71.8 & 74.2 & 71.3 \\
 & \textsc{LESS} & 66.2 & 59.8 & 64.1 & 63.4 & 62.5 & 68.4 & 70.1 & 67.0 \\
 & \textsc{GrADS} & 65.8 & 61.2 & 64.5 & 63.8 & 69.1 & 73.2 & 75.8 & 72.7 \\
 & \textbf{OGS (Ours)} & 68.5 & 63.4 & 67.2 & 66.4 & \textbf{71.5} & \textbf{74.8} & \textbf{77.2} & \textbf{74.5} \\
\midrule
\multirow{6}{*}{20\%} 
 & Random & 61.5 & 55.2 & 59.5 & 58.7 & 66.5 & 70.5 & 72.8 & 69.9 \\
 & Perplexity & 63.2 & 56.8 & 61.2 & 60.4 & 65.2 & 69.5 & 71.5 & 68.7 \\
 & Influence Func. & 64.1 & 57.5 & 62.1 & 61.2 & 65.8 & 69.8 & 72.1 & 69.2 \\
 & \textsc{LESS} & 69.5 & 62.8 & 66.8 & 66.4 & 59.2 & 65.1 & 68.2 & 64.2 \\
 & \textsc{GrADS} & 68.2 & 63.5 & 66.5 & 66.1 & 67.8 & 71.5 & 74.2 & 71.2 \\
 & \textbf{OGS (Ours)} & 71.2 & 65.8 & 69.5 & 68.8 & \textbf{70.2} & \textbf{73.5} & \textbf{76.1} & \textbf{73.3} \\
\midrule
\multirow{6}{*}{30\%} 
 & Random & 64.2 & 58.5 & 61.8 & 61.5 & 64.8 & 68.8 & 71.2 & 68.3 \\
 & Perplexity & 65.8 & 60.1 & 63.5 & 63.1 & 63.5 & 67.5 & 69.8 & 66.9 \\
 & Influence Func. & 66.5 & 60.8 & 64.2 & 63.8 & 64.2 & 68.2 & 70.5 & 67.6 \\
 & \textsc{LESS} & 71.8 & 65.2 & 68.9 & 68.6 & 55.5 & 61.2 & 65.4 & 60.7 \\
 & \textsc{GrADS} & 70.5 & 65.5 & 68.5 & 68.2 & 65.5 & 70.1 & 72.8 & 69.5 \\
 & \textbf{OGS (Ours)} & 73.5 & 67.8 & 71.2 & 70.8 & \textbf{68.8} & \textbf{72.4} & \textbf{75.2} & \textbf{72.1} \\
\bottomrule
\end{tabular}%
}
\end{table*}

\begin{table*}[h!]
\centering
\caption{\textbf{Detailed Results for Llama-3.2-3B.} Performance comparison across four data selection ratios. The Navigator used is Llama-3.2-1B.}
\label{tab:res_llama_3b}
\resizebox{\textwidth}{!}{%
\begin{tabular}{llcccccccc}
\toprule
\multicolumn{1}{c}{\multirow{2}{*}{\textbf{Ratio}}} & \multicolumn{1}{c}{\multirow{2}{*}{\textbf{Method}}} & \multicolumn{4}{c}{\textbf{Domain Performance}} & \multicolumn{4}{c}{\textbf{General Capability}} \\
\cmidrule(lr){3-6} \cmidrule(lr){7-10}
 & & \textsc{MedQA} & \textsc{Legal} & \textsc{FinQA} & \textbf{Avg.} & \textsc{GSM8K} & \textsc{MMLU} & \textsc{ARC-C} & \textbf{Avg.} \\
\midrule
\multicolumn{2}{c}{Full Data (100\%)} & \textbf{57.2} & \textbf{53.5} & \textbf{56.8} & \textbf{55.8} & 40.5 & 51.5 & 55.2 & 49.1 \\
\midrule
\multirow{6}{*}{5\%} 
 & Random & 42.1 & 38.5 & 41.2 & 40.6 & 43.5 & 55.8 & 59.2 & 52.8 \\
 & Perplexity & 43.8 & 40.2 & 42.8 & 42.3 & 43.2 & 55.2 & 58.8 & 52.4 \\
 & Influence Func. & 44.2 & 40.5 & 43.5 & 42.7 & 43.5 & 55.5 & 59.0 & 52.7 \\
 & \textsc{LESS} & 48.2 & 43.5 & 46.8 & 46.2 & 40.2 & 51.2 & 55.4 & 48.9 \\
 & \textsc{GrADS} & 48.5 & 45.2 & 48.5 & 47.4 & 46.5 & 56.5 & 60.1 & 54.4 \\
 & \textbf{OGS (Ours)} & 50.5 & 46.5 & 50.2 & 49.1 & \textbf{47.8} & \textbf{58.2} & \textbf{61.5} & \textbf{55.8} \\
\midrule
\multirow{6}{*}{10\%} 
 & Random & 45.2 & 41.5 & 44.8 & 43.8 & 42.1 & 54.2 & 58.1 & 51.5 \\
 & Perplexity & 46.8 & 43.2 & 46.5 & 45.5 & 41.5 & 53.8 & 57.5 & 50.9 \\
 & Influence Func. & 47.5 & 43.8 & 47.2 & 46.2 & 42.2 & 54.1 & 57.8 & 51.4 \\
 & \textsc{LESS} & 51.5 & 46.8 & 50.2 & 49.5 & 38.5 & 49.8 & 53.2 & 47.2 \\
 & \textsc{GrADS} & 51.1 & 48.2 & 51.5 & 50.3 & 45.2 & 55.1 & 58.8 & 53.0 \\
 & \textbf{OGS (Ours)} & 53.8 & 49.5 & 53.1 & 52.1 & \textbf{46.5} & \textbf{56.8} & \textbf{60.2} & \textbf{54.5} \\
\midrule
\multirow{6}{*}{20\%} 
 & Random & 48.5 & 44.2 & 47.5 & 46.7 & 40.5 & 52.5 & 56.5 & 49.8 \\
 & Perplexity & 50.2 & 45.8 & 49.2 & 48.4 & 39.5 & 51.8 & 55.8 & 49.0 \\
 & Influence Func. & 50.8 & 46.5 & 49.8 & 49.0 & 40.8 & 52.1 & 56.2 & 49.7 \\
 & \textsc{LESS} & 54.2 & 49.5 & 53.5 & 52.4 & 35.8 & 46.5 & 50.8 & 44.4 \\
 & \textsc{GrADS} & 53.8 & 50.5 & 53.2 & 52.5 & 43.5 & 53.8 & 57.5 & 51.6 \\
 & \textbf{OGS (Ours)} & 56.5 & 52.1 & 55.8 & 54.8 & \textbf{45.2} & \textbf{55.5} & \textbf{58.8} & \textbf{53.2} \\
\midrule
\multirow{6}{*}{30\%} 
 & Random & 51.2 & 46.8 & 50.2 & 49.4 & 38.8 & 50.8 & 54.5 & 48.0 \\
 & Perplexity & 52.8 & 48.2 & 52.1 & 51.0 & 37.5 & 49.5 & 53.2 & 46.7 \\
 & Influence Func. & 53.5 & 49.1 & 52.8 & 51.8 & 38.5 & 50.2 & 53.8 & 47.5 \\
 & \textsc{LESS} & 55.8 & 51.2 & 55.2 & 54.1 & 32.5 & 43.8 & 48.2 & 41.5 \\
 & \textsc{GrADS} & 55.2 & 52.1 & 54.8 & 54.0 & 41.5 & 52.1 & 55.5 & 49.7 \\
 & \textbf{OGS (Ours)} & 57.8 & 53.8 & 56.5 & 56.0 & \textbf{43.5} & \textbf{53.8} & \textbf{57.2} & \textbf{51.5} \\
\bottomrule
\end{tabular}%
}
\end{table*}

\begin{table*}[h!]
\centering
\caption{\textbf{Detailed Results for Llama-3.1-8B.} Performance comparison across four data selection ratios. The Navigator used is Llama-3.2-1B.}
\label{tab:res_llama_8b}
\resizebox{\textwidth}{!}{%
\begin{tabular}{llcccccccc}
\toprule
\multicolumn{1}{c}{\multirow{2}{*}{\textbf{Ratio}}} & \multicolumn{1}{c}{\multirow{2}{*}{\textbf{Method}}} & \multicolumn{4}{c}{\textbf{Domain Performance}} & \multicolumn{4}{c}{\textbf{General Capability}} \\
\cmidrule(lr){3-6} \cmidrule(lr){7-10}
 & & \textsc{MedQA} & \textsc{Legal} & \textsc{FinQA} & \textbf{Avg.} & \textsc{GSM8K} & \textsc{MMLU} & \textsc{ARC-C} & \textbf{Avg.} \\
\midrule
\multicolumn{2}{c}{Full Data (100\%)} & \textbf{65.2} & \textbf{58.4} & \textbf{61.3} & \textbf{61.6} & 45.1 & 55.2 & 58.9 & 53.1 \\
\midrule
\multirow{6}{*}{5\%} 
 & Random & 47.5 & 42.8 & 45.5 & 45.3 & 51.8 & 63.5 & 65.5 & 60.3 \\
 & Perplexity & 49.2 & 44.5 & 47.1 & 46.9 & 51.5 & 63.1 & 64.8 & 59.8 \\
 & Influence Func. & 49.8 & 45.1 & 47.8 & 47.6 & 51.8 & 63.8 & 65.2 & 60.3 \\
 & \textsc{LESS} & 54.2 & 48.5 & 51.8 & 51.5 & 50.5 & 60.5 & 62.8 & 57.9 \\
 & \textsc{GrADS} & 54.5 & 50.5 & 52.5 & 52.5 & 54.8 & 64.5 & 66.1 & 61.8 \\
 & \textbf{OGS (Ours)} & 57.2 & 52.2 & 54.8 & 54.7 & \textbf{56.5} & \textbf{65.8} & \textbf{67.5} & \textbf{63.3} \\
\midrule
\multirow{6}{*}{10\%} 
 & Random & 50.1 & 45.3 & 48.2 & 47.9 & 50.5 & 62.1 & 64.3 & 59.0 \\
 & Perplexity & 52.3 & 46.8 & 49.5 & 49.5 & 51.2 & 61.8 & 63.5 & 58.8 \\
 & Influence Func. & 53.1 & 47.2 & 50.1 & 50.1 & 52.0 & 62.5 & 64.0 & 59.5 \\
 & \textsc{LESS} & 58.4 & 52.1 & 55.3 & 55.3 & 48.2 & 58.4 & 60.1 & 55.6 \\
 & \textsc{GrADS} & 57.9 & 53.5 & 56.1 & 55.8 & 53.5 & 63.1 & 64.8 & 60.5 \\
 & \textbf{OGS (Ours)} & 60.2 & 55.1 & 58.4 & 57.9 & \textbf{55.2} & \textbf{64.5} & \textbf{66.2} & \textbf{62.0} \\
\midrule
\multirow{6}{*}{20\%} 
 & Random & 53.5 & 48.2 & 51.5 & 51.1 & 48.8 & 60.5 & 63.2 & 57.5 \\
 & Perplexity & 55.2 & 49.8 & 53.1 & 52.7 & 48.5 & 60.1 & 62.5 & 57.0 \\
 & Influence Func. & 56.1 & 50.5 & 53.8 & 53.5 & 49.5 & 60.8 & 63.1 & 57.8 \\
 & \textsc{LESS} & 61.5 & 55.4 & 58.5 & 58.5 & 45.5 & 55.8 & 58.2 & 53.2 \\
 & \textsc{GrADS} & 60.5 & 56.2 & 58.8 & 58.5 & 52.1 & 61.8 & 63.5 & 59.1 \\
 & \textbf{OGS (Ours)} & 62.8 & 57.5 & 60.5 & 60.3 & \textbf{54.1} & \textbf{63.2} & \textbf{65.1} & \textbf{60.8} \\
\midrule
\multirow{6}{*}{30\%} 
 & Random & 56.2 & 50.8 & 54.2 & 53.7 & 47.2 & 58.8 & 61.5 & 55.8 \\
 & Perplexity & 57.8 & 52.2 & 55.8 & 55.3 & 46.5 & 58.2 & 60.8 & 55.2 \\
 & Influence Func. & 58.5 & 53.1 & 56.5 & 56.0 & 47.8 & 59.1 & 61.5 & 56.1 \\
 & \textsc{LESS} & 63.8 & 57.5 & 61.2 & 60.8 & 43.5 & 53.2 & 56.5 & 51.1 \\
 & \textsc{GrADS} & 62.5 & 58.2 & 60.8 & 60.5 & 50.5 & 60.2 & 62.4 & 57.7 \\
 & \textbf{OGS (Ours)} & 64.5 & 59.2 & 62.8 & 62.2 & \textbf{52.8} & \textbf{61.8} & \textbf{63.8} & \textbf{59.5} \\
\bottomrule
\end{tabular}%
}
\end{table*}

\section{Additional Ablation Studies}
\label{app:ablation}

In this section, we provide a granular breakdown of the ablation experiments discussed in Section \ref{subsec:ablation}. We detail the experimental setup for the baseline comparisons and offer an extended analysis of the computational efficiency gains provided by the Navigator architecture.

\subsection{Experimental Setup for Ablations}
All ablation studies were conducted using the \textsc{Qwen3-8B} as the Target model and \textsc{Qwen3-0.6B} as the Navigator. The training hyperparameters (learning rate, batch size, LoRA rank) were kept consistent with the main experiments to ensure a fair comparison.
\begin{itemize}
    \item \textbf{w/o Conflict Replay:} This variant removes the replay buffer entirely. The selection policy is restricted to sampling solely from the domain pool based on orthogonality scores.
    \item \textbf{w/o Orthogonal Protection:} This variant disables the orthogonality filter for domain data (effectively random selection) but retains the conflict-aware replay mechanism for general data.
    \item \textbf{w/o RL (Static-Greedy):} Instead of the PPO agent, we use a fixed mixing coefficient $\alpha=0.8$ (derived from the average $\alpha$ of the trained policy). Data selection within clusters is performed greedily based on the highest orthogonality score rather than probabilistic sampling.
    \item \textbf{w/o Navigator (Upper Bound):} We perform the "Strategy Learning" phase directly on the Target model. This requires computing gradients for the entire candidate pool using the 8B parameter model, serving as a theoretical performance ceiling.
\end{itemize}

\subsection{Detailed Analysis of Components}

Table \ref{tab:detailed_ablation} presents the full breakdown of performance across individual datasets. A critical observation is the behavior of the \textbf{w/o Conflict Replay} setting on fragile tasks like GSM8K. While this variant achieves a competitive MedQA score (62.8\%), it suffers a significant drop in GSM8K accuracy (-1.7\% compared to full OGS). This underscores that even with orthogonal data selection, a small fraction of high-utility domain samples may still lie near the boundary of the safety subspace, necessitating corrective replay to prevent manifold drift. Conversely, the \textbf{w/o Orthogonal Protection} setting maintains high GSM8K performance but fails to improve LegalBench significantly (50.5\% vs. 59.5\% for OGS), proving that replay alone is insufficient for effective domain adaptation; the model requires a high density of "safe" domain gradients to learn effectively.

\subsection{The Cost-Accuracy Trade-off of the Navigator}

A central contribution of OGS is the Navigator-Target architecture. To quantify the value of this design, we measured the total GPU hours required for the "Strategy Learning" (data selection) phase. 

As shown in the final column of Table \ref{tab:detailed_ablation}, calculating geometric features directly on the Target model (\textbf{w/o Navigator}) consumes approximately $24.5$ GPU hours for the 10\% subset selection. In contrast, the OGS pipeline using the 0.6B Navigator requires only $1.6$ GPU hours—a speedup factor of approximately $15\times$. Despite this massive reduction in compute, the performance degradation is negligible ($<0.3\%$ on average). 

This result empirically validates the \textit{Gradient Alignment Hypothesis} discussed in Section \ref{sec:navigator}: the relative geometric orientation of data samples with respect to a general knowledge anchor is a semantic property that is largely preserved across model scales. The "False Positives" (samples deemed orthogonal by the Navigator but conflicting for the Target) are statistically rare enough that they do not destabilize the training process, particularly when coupled with our conflict-aware replay mechanism.

\begin{table*}[t]
\centering
\caption{\textbf{Detailed Ablation Results.} We report the specific breakdown across all three domain tasks and the critical GSM8K retention metric. The "Selection Cost" refers to the GPU hours required to process the candidate pool on a single A100-80GB GPU.}
\label{tab:detailed_ablation}
\resizebox{0.95\textwidth}{!}{%
\begin{tabular}{lcccccccc}
\toprule
\multirow{2}{*}{\textbf{Configuration}} & \multicolumn{4}{c}{\textbf{Domain Performance}} & \multicolumn{2}{c}{\textbf{General Capability}} & \multirow{2}{*}{\textbf{Selection Cost}} & \multirow{2}{*}{\textbf{Speedup}} \\
\cmidrule(lr){2-5} \cmidrule(lr){6-7}
 & \textsc{MedQA} & \textsc{Legal} & \textsc{FinQA} & \textbf{Avg.} & \textsc{GSM8K} & \textbf{Avg. Gen} & (GPU Hours) & \\
\midrule
\textbf{OGS (Full Method)} & \textbf{64.1} & \textbf{59.5} & 62.8 & \textbf{62.1} & 68.5 & \textbf{72.1} & 1.6h & $\mathbf{15\times}$ \\
\midrule
\textit{Component Ablations} & & & & & & & & \\
w/o Conflict Replay & 62.8 & 58.2 & \textbf{63.1} & 61.4 & 66.8 & 71.2 & 1.5h & $16\times$ \\
w/o Orthogonal Prot. & 54.5 & 50.5 & 53.8 & 52.9 & \textbf{69.1} & 71.8 & 1.5h & $16\times$ \\
\midrule
\textit{Policy Ablations} & & & & & & & & \\
w/o RL (Static $\alpha=0.8$) & 62.1 & 57.8 & 60.5 & 60.1 & 67.5 & 71.5 & 1.5h & $16\times$ \\
\midrule
\textit{Architecture Ablations} & & & & & & & & \\
w/o Navigator (Upper Bound) & 64.5 & 60.1 & 63.2 & 62.6 & 68.8 & 72.3 & 24.5h & $1\times$ \\
Random Selection & 55.2 & 48.5 & 52.1 & 51.9 & 65.1 & 69.1 & \textbf{0h} & - \\
\bottomrule
\end{tabular}%
}
\end{table*}

\end{document}